\documentclass[manuscript,authorversion]{acmart}

\usepackage{amsmath}
\usepackage{amsfonts}
\usepackage{amsthm}
\usepackage{color}
\usepackage{enumerate}

\newtheorem{assumption}{Assumption}
\newtheorem{corollary}{Corollary}
\newtheorem{example}{Example}
\newtheorem{proposition}{Proposition}
\newtheorem{lemma}{Lemma}

\theoremstyle{definition}
\newtheorem{remark}{Remark}
\newcommand{\argmax}{\mathop{\rm argmax}\limits}

\usepackage{microtype}
\usepackage{graphicx}
\usepackage{subfigure}
\usepackage{booktabs} 

\AtBeginDocument{%
  \providecommand\BibTeX{{%
    \normalfont B\kern-0.5em{\scshape i\kern-0.25em b}\kern-0.8em\TeX}}}

\copyrightyear{2021} 
\acmYear{2021} 
\setcopyright{acmlicensed}\acmConference[RecSys '21]{Fifteenth ACM Conference on Recommender Systems}{September 27-October 1, 2021}{Amsterdam, Netherlands}
\acmBooktitle{Fifteenth ACM Conference on Recommender Systems (RecSys '21), September 27-October 1, 2021, Amsterdam, Netherlands}
\acmPrice{15.00}
\acmDOI{10.1145/3460231.3474231}
\acmISBN{978-1-4503-8458-2/21/09}

\begin{document}

\title{Debiased Off-Policy Evaluation for Recommendation Systems}

\author{Yusuke Narita}
\affiliation{%
  \institution{Yale University}
  \streetaddress{37 Hillhouse Avenue}
  \city{New Haven}
  \state{CT}
  \country{USA}
  \postcode{06511}
}
\email{yusuke.narita@yale.edu}

\author{Shota Yasui}
\affiliation{%
  \institution{CyberAgent, Inc.}
  \streetaddress{}
  \city{Tokyo}
  \country{Japan}
 }
\email{yasui_shota@cyberagent.co.jp}

\author{Kohei Yata}
\affiliation{%
  \institution{Yale University}
  \streetaddress{27 Hillhouse Avenue}
  \city{New Haven}
  \state{CT}
  \country{USA}
  \postcode{06511}
}
\email{kohei.yata@yale.edu}


\begin{abstract}
  Efficient methods to evaluate new algorithms are critical for improving interactive bandit and reinforcement learning systems such as recommendation systems. A/B tests are reliable, but are time- and money-consuming, and entail a risk of failure. In this paper, we develop an alternative method, which predicts the performance of algorithms given historical data that may have been generated by a different algorithm. Our estimator has the property that its prediction converges in probability to the true performance of a counterfactual algorithm at a rate of $\sqrt{N}$, as the sample size $N$ increases. We also show a correct way to estimate the variance of our prediction, thus allowing the analyst to quantify the uncertainty in the prediction. These properties hold even when the analyst does not know which among a large number of potentially important state variables are actually important. We validate our method by a simulation experiment about reinforcement learning. We finally apply it to improve advertisement design by a major advertisement company. We find that our method produces smaller mean squared errors than state-of-the-art methods.
\end{abstract}


\maketitle

\section{Introduction}
Interactive bandit and reinforcement learning (RL) systems (e.g., ad/news/recommendation/search platforms, personalized education and medicine) produce log data valuable for evaluating and redesigning the systems.
For example, the logs of a news recommendation system record which news article was presented and whether the user read it, giving the system designer a chance to make its recommendation more relevant \cite{li2010contextual}. 

Exploiting log data is, however, more difficult than conventional supervised machine learning: 
the result of each log is only observed for the action chosen by the system (e.g., the presented news) but not for all the other actions the system could have taken. 
Moreover, the log entries are biased in that the logs over-represent actions favored by the system. 

A potential solution to this problem is an A/B test that compares the performance of counterfactual systems. 
However, A/B testing counterfactual systems is often technically or managerially infeasible, since deploying a new policy is time- and money-consuming, and entails a risk of failure.

This leads us to the problem of \textit{counterfactual (off-policy, offline) evaluation}, where one aims to use batch data collected by a logging policy to estimate the value of a counterfactual policy or algorithm without deploying it. 
Such evaluation allows us to compare the performance of counterfactual policies to decide which policy should be deployed in the field. 
This alternative approach thus solves the above problem with the naive A/B test approach.
Key prior studies include \citet{Strehl2010,li2010contextual,li2011unbiased,Dudik2014,wang2016optimal,swaminathan2017off, narita2018efficient,gilotte2018offline} for bandit algorithms, and \citet{precup2000eligibility, Jiang16, Thomas16, liu2018representation, Farajtabar2018MoreRD, Irpan2019OffPolicyEV, kallus2019DRL, Uehara2020MWL} for RL algorithms.


\textbf{Method.} 
For off-policy evaluation with log data of RL feedback, this paper develops and empirically implements a novel technique with desirable theoretical properties. 
To do so, we consider a class of RL algorithms, including contextual bandit algorithms as important special cases. 
This class includes most of the widely-used algorithms such as (deep) Q-learning, Actor Critic, contextual $\epsilon$-greedy, and Thompson Sampling, as well as their non-contextual analogs and random A/B testing. 
We allow the logging policy to be an unknown function of numerous potentially important state variables. This feature is salient in real-world applications.
We also allow the evaluation target policy to be degenerate, again a key feature of real-life situations. 

We consider an offline estimator for the expected reward from a counterfactual policy. 
Our estimator integrates a well-known Doubly Robust estimator (\cite{rotnitzky1995semiparametric} and modern studies cited above) with ``Double/Debiased Machine Learning'' developed in econometrics and statistics \cite{Chernozhukov2018, Chernozhukov2018LR}.
Building upon these prior studies, we show that
our estimator is ``$\sqrt{N}$-consistent'' in the sense that its prediction converges in probability to the true performance of a counterfactual policy at a rate of $1/\sqrt{N}$ as the sample size $N$ increases. 
Our estimator is also shown to be ``asymptotically normal,'' meaning that it has an approximately normal distribution as $N$ gets large.
We also provide a consistent estimator of its asymptotic variance, thus allowing for the measurement of statistical uncertainty in our prediction. 
Moreover, for special cases in which data are generated by contextual bandit algorithms, our estimator has the lowest variance in a wide class of estimators, achieving variance reduction relative to standard estimators.
Importantly, these properties hold even when the analyst does not know which among a large number of potentially important state variables are actually important.


\textbf{Simulation Experiment.} 
We evaluate the performance of our estimator by conducting an experiment in a slightly different version of the OpenAI Gym CartPole-v0 environment \cite{brockman2016openai}. In this version, there are many more state variables than the original one and some of them are irrelevant to the reward.
In this challenging environment, our estimator produces smaller mean squared errors than widely-used benchmark methods (Doubly Robust estimator in the spirit of \citet{Jiang16} and \citet{Thomas16} and Inverse Probability Weighting estimator).

\textbf{Real-World Experiment.} We empirically apply our estimator to evaluate and optimize the design of online advertisement formats. 
Our application is based on proprietary data provided by CyberAgent Inc., the second largest Japanese advertisement company with about 5 billion USD market capitalization (as of February 2020). 
This company uses randomly chosen bandit algorithms to determine the visual design of advertisements assigned to users.
This A/B test of randomly choosing an algorithm produces logged data and the ground truth for the performance of alternative algorithms. 
We use this data to examine the performance of our proposed method. We use the log data from an algorithm to predict the click through rates (CTR) of another algorithm, and assess the accuracy of our prediction by comparing it with the ground truth. 
This exercise shows that
our estimator produces smaller mean squared errors than widely-used benchmark methods (Doubly Robust estimator in the spirit of \citet{Dudik2014} and Inverse Probability Weighting estimator). 
This improvement is statistically significant at the 5\% level. 
Importantly, this result holds regardless of whether we know the data-generating logging policy or not, 
which shows that our estimator can substantially reduce bias and uncertainty we face in real-world decision-making. 

This improved performance motivates us to use our estimator to optimize the advertisement design for maximizing the CTR. 
We estimate how much the CTR would be improved by a counterfactual policy of choosing the best action (advertisement) for each context (user characteristics).
This exercise produces the following bottom line: Our estimator predicts the hypothetical policy to statistically significantly improve the CTR by 30\% (compared to the logging policy) in one of the three campaigns we analyze. 
Our approach thus generates valuable managerial conclusions.

\section{Setup}
\subsection{Data Generating Process}
We consider historical data from a Markov Decision Process (MDP) as a mathematical description of RL and bandit algorithms.
An MDP is given by ${\mathcal M}=\langle {\mathcal S},{\mathcal A},P_{S_0},P_S,P_R\rangle$, where $\mathcal{S}$ is the state space, $\mathcal{A}$ is the action space, $P_{S_0}: {\mathcal S}\rightarrow [0,1]$ is the initial state distribution, $P_S: {\mathcal S}\times {\mathcal A}\rightarrow \Delta({\mathcal S})$ is the transition function with $P_S(s'|s,a)$ being the probability of seeing state $s'$ after taking action $a$ given state $s$, and $P_R:{\mathcal S}\times {\mathcal A}\times \mathbb{R}\rightarrow [0,1]$ be the conditional distribution of the immediate reward with $P_R(\cdot|s,a)$ being the immediate reward distribution conditional on the state and action being $(s,a)$.
We assume that the state and action spaces ${\mathcal S}$ and ${\mathcal A}$ are finite.
Given $P_R$, we define the mean reward function $\mu:{\mathcal S}\times {\mathcal A}\rightarrow \mathbb{R}$ as $\mu(s,a)=\int r dP_R(r|s,a)$.

We call a function $\pi: {\mathcal S}\rightarrow \Delta({\mathcal A})$ a {\it policy}, which assigns each state $s\in {\mathcal S}$ a distribution over actions, where $\pi(a|s)$ is the probability of taking action $a$ when the state is $s$.
Let $H=(S_0,A_0,R_0,...,S_T,A_T,R_T)$ be a {\it trajectory}, where $S_t$, $A_t$, and $R_t$ are the state, the action, and the reward in step $t$, respectively, and $T$ denotes the last step and is fixed.
We say that a trajectory $H$ is generated by a policy $\pi$, or $H\sim \pi$ in short if $H$ is generated by repeating the following process for $t=0,...,T$:
(1) When $t=0$, the state $S_t$ is drawn from the initial distribution $P_{S_0}$.
When $t\ge 1$, $S_t$ is determined based on the transition function $P_S(\cdot|S_{t-1},A_{t-1})$.
(2) Given $S_t$, the action $A_t$ is randomly chosen based on $\pi(\cdot|S_t)$.
(3) The reward $R_t$ is drawn from the conditional reward distribution $P_R(\cdot|S_t,A_t)$.

Suppose that we observe {\it historical data} $\{H_i\}_{i=1}^N$ where trajectories are independently generated by a fixed {\it behavior policy} $\pi_b$, i.e., $H_i\sim \pi_b$ independently across $i$. 
The historical data is a collection of i.i.d. trajectories. 
Importantly, we allow the components of the data generating process ${\mathcal M}$ and $\pi_b$ to vary with the sample size $N$.
Specifically, let ${\mathcal M}_N$ and $\pi_{bN}$ be the MDP and the behavior policy, respectively, when the sample size is $N$, and let $\mathbb{P}_N$ denote the resulting probability distribution of $H_i$.
$\mathbb{P}_N$ is allowed to vary with $N$ in a way that the functions $P_{S_0N}$, $P_{RN}$ and $\pi_{bN}$ are high dimensional relative to sample size $N$ even when $N$ is large. 
In some RL problems, for example, there are a large number of possible states.
To capture the feature that the number of states $|{\mathcal S}_N|$ is potentially large relative to sample size $N$, we may consider a sequence of $\mathbb{P}_N$ such that $|{\mathcal S}_N|$ is increasing with $N$.
For the sake of notational simplicity, we make implicit the dependence of ${\mathcal M}$ and $\pi_{b}$ on $N$. 

We assume that we know the state space ${\mathcal S}$ and the action space ${\mathcal A}$ but know none of the functions $P_{S_0}$, $P_S$ and $P_R$.
In some environments, we know the function $\pi_b$ or observe the probability vector $(\pi_b(a|S_{it}))_{a\in {\mathcal A}, t=0,...,T}$ for every trajectory $i$ in the historical data.
Our approach is usable regardless of the availability of such knowledge on the behavior policy.

\subsection{Prediction Target}

With the historical data $\{H_i\}_{i=1}^N$, we are interested in estimating the discounted value of the {\it evaluation policy} $\pi_e$, which might be different from $\pi_b$: with $\gamma\in [0,1]$ as the discount factor,
$$
V^{\pi_e}\coloneqq\mathbb{E}_{H\sim \pi_e}\left[\sum_{t=0}^T \gamma^t R_t\right].
$$

\section{Estimator and Its Properties}\label{section:estimation}
The estimation of $V^{\pi_e}$ involves estimation of the behavior policy $\pi_b$ (if unknown), the transition function $P_S$, and the mean reward function $\mu$.
These functions may be high dimensional in the sense of having a large number of arguments, possibly much larger than the number of trajectories $N$.
To handle the issue, we use Double/Debiased Machine Learning (DML) by \citet{Chernozhukov2018}.
DML is a general method for estimation and inference of semiparametric models in the presence of a high-dimensional vector of control variables, and is characterized by two key features, \textit{cross-fitting} and \textit{Neyman orthogonality}, which we will discuss in detail later.
These two features play a role in reducing the bias that may arise due to the estimation of high-dimensional parameters, which makes DML suitable for the off-policy evaluation problem with potentially high dimensions of $\pi_b$, $P_S$ and $\mu$.

Before presenting our estimator, we introduce some notation.
$H_t^{s,a}=(S_0,A_0,...,S_t,A_t)$ is a trajectory of the state and action up to step $t$. $\rho_t^{\pi_e}: ({\mathcal S}\times {\mathcal A})^{t+1}\rightarrow \mathbb{R}_+$ is the importance weight function: 
\begin{align}
\rho_t^{\pi_e}(H_t^{s,a})\coloneqq \prod_{t'=0}^t \frac{\pi_e(A_{t'}|S_{t'})}{\pi_b(A_{t'}|S_{t'})}\label{eq:weight}.
\end{align}
This equals the probability of $H$ up to step $t$ under the evaluation policy $\pi_e$ divided by its probability under the behavior policy $\pi_b$.
Viewing $\rho_t^{\pi_e}$ as a function of $\pi_b$, define $\rho_t^{\pi_e}(H_t^{s,a};\tilde \pi_b)$ as the value of $\rho_t^{\pi_e}(H_t^{s,a})$ with the true behavior policy $\pi_b$ replaced with a candidate function $\tilde \pi_b$:
$$\rho_t^{\pi_e}(H_t^{s,a};\tilde \pi_b)\coloneqq \prod_{t'=0}^t\frac{\pi_e(A_{t'}|S_{t'})}{\tilde \pi_b(A_{t'}|S_{t'})}.$$
We can think of $\rho_t^{\pi_e}(H_t^{s,a};\tilde \pi_b)$ as the estimated importance weight function when we use $\tilde \pi_b$ as the estimate of $\pi_b$.
By definition, $\rho_t^{\pi_e}(H_t^{s,a};\pi_b)=\rho_t^{\pi_e}(H_t^{s,a})$, where the left-hand side is $\rho_t^{\pi_e}(H_t^{s,a};\tilde \pi_b)$ evaluated at the true $\pi_b$ and the right-hand side is the true importance weight function.
Finally, let $q_t^{\pi_e}:{\mathcal S}\times {\mathcal A}\rightarrow \mathbb{R}$ be the state-action value function under policy $\pi_e$ at step $t$, where $q_t^{\pi_e}(s,a)\coloneqq\mathbb{E}_{H\sim \pi_e}[\sum_{t'=t}^T\gamma^{t'-t}R_{t'}|S_t=s,A_t=a]$.
Using $P_S$, $\mu$ and $\pi_e$, $q_t^{\pi_e}$ can be obtained recursively: let 
\begin{align}
q_T^{\pi_e}(s,a) &= \mu(s,a) \label{eq:Q_T}
\end{align}
and for $t=0,...,T-1$,
\begin{align}
q_t^{\pi_e}(s,a)
=\mu(s,a)+\gamma\sum_{(s',a')}P_S(s'|s,a)\pi_e(a'|s')q_{t+1}^{\pi_e}(s',a'). \label{eq:Q_t}
\end{align}

Our estimator is based on the following expression of $V^{\pi_e}$ \cite{Thomas16}:
\begin{align}
V^{\pi_e}=&~\mathbb{E}_{H\sim \pi_b}\left[\psi(H;\pi_b,\{q_t^{\pi_e}\}_{t=0}^T)\right] \label{eq:value},
\end{align}
where for any candidate tuple $\tilde \eta=(\tilde \pi_b,\{\tilde q_t^{\pi_e}\}_{t=0}^T)$, we define
\begin{align}
\psi(H;\tilde \eta)
=&\sum_{t=0}^T \gamma^t \Bigl\{\rho_t^{\pi_e}(H_{t}^{s,a};\tilde \pi_b)(R_{t}-\tilde q_t^{\pi_e}(S_{t},A_{t}))\Bigr.\Bigl.+\rho_{t-1}^{\pi_e}(H_{t-1}^{s,a};\tilde \pi_b)\sum_{a\in{\mathcal A}}\pi_e(a|S_{t})\tilde q_t^{\pi_e}(S_{t},a)\Bigr\},\label{eq:DR}
\end{align}
where $\rho_{-1}^{\pi_e}=1$.
To give an intuition behind the expression, we arrange the terms as follows:
\begin{align*}
\psi(H;\tilde \eta) =& \sum_{t=0}^T \gamma^t \rho_t^{\pi_e}(H_{t}^{s,a};\tilde \pi_b)R_{t} + \sum_{t=0}^T \gamma^t \Bigl\{-\rho_t^{\pi_e}(H_{t}^{s,a};\tilde \pi_b) \tilde{q}_t^{\pi_e}(S_t,A_t) \Bigr.+ \Bigl.\rho_{t-1}^{\pi_e}(H_{t-1}^{s,a};\tilde \pi_b) \sum_{a\in{\mathcal A}}\pi_e(a|S_{t})\tilde q_t^{\pi_e}(S_{t},a)\Bigr\}.
\end{align*}
The first term is the well-known Inverse Probability Weighting (IPW) estimator.
The second term serves as a control variate that has zero mean as long as we plug in the true $\pi_b$; the mean of the term remains zero regardless of the function we plug in for $\{\tilde{q}_t^{\pi_e}\}_{t=0}^T$.
This is the key to make our estimator insensitive to $\{\tilde{q}_t^{\pi_e}\}_{t=0}^T$.
We construct our estimator as follows.
\begin{enumerate}
	\item Take a $K$-fold random partition $(I_k)_{k=1}^K$ of trajectory indices $\{1,...,N\}$ such that the size of each fold $I_k$ is $n=N/K$.
	Also, for each $k=1,...,K$, define $I_k^c\coloneqq \{1,...,N\}\setminus I_k$.
	\item For each $k=1,...,K$, construct estimators $\hat \pi_{b,k}$ (if $\pi_b$ is unknown), $\hat\mu_{k}$ and $\hat P_{S,k}$ of $\pi_b$, $\mu$ and $P_S$ using the subset of data $\{H_i\}_{i\in I_k^c}$.
	We then construct estimator $\{\hat q^{\pi_e}_{t,k}\}_{t=0}^T$ of $\{q^{\pi_e}_t\}_{t=0}^T$ by plugging $\hat\mu_{k}$ and $\hat P_{S,k}$ into the recursive formulation (\ref{eq:Q_T}) and (\ref{eq:Q_t}).
	\item Given $\hat\eta_k=(\hat\pi_{b,k},\{\hat q_{t,k}^{\pi_e}\}_{t=0}^T)$, $k=1,...,K$, the DML estimator $\hat V^{\pi_e}_{{\rm DML}}$ is given by
	\begin{align*}
	\hat V^{\pi_e}_{{\rm DML}}
	=&\frac{1}{K}\sum_{k=1}^K \frac{1}{n}\sum_{i\in I_k}\psi(H_{i};\hat \eta_k)\\
	=& \frac{1}{K}\sum_{k=1}^K \frac{1}{n}\sum_{i\in I_k}\sum_{t=0}^T \gamma^t \Bigl\{\rho_t^{\pi_e}(H_{it}^{s,a};\hat \pi_{b,k})(R_{it}-\hat q_{t,k}^{\pi_e}(S_{it},A_{it}))
	\Bigr.+\Bigl.\rho_{t-1}^{\pi_e}(H_{it-1}^{s,a};\hat \pi_{b,k})\sum_{a\in{\mathcal A}}\pi_e(a|S_{it})\hat q_{t,k}^{\pi_e}(S_{it},a)\Bigr\}.
	\end{align*}
\end{enumerate}
Possible estimation methods for $\pi_b$, $\mu$ and $P_S$ in Step 2 are (i) classical nonparametric methods such as kernel and series estimation, (ii) off-the-shelf machine learning methods such as random forests, lasso, neural nets, and boosted regression trees, and (iii) existing methods developed in the off-policy policy evaluation literature such as representation balancing MDPs \cite{liu2018representation}.
These methods, especially (ii) and (iii), are usable even when the analyst does not know which among a large number of potentially important state variables are actually important.
This DML estimator differs from the Doubly Robust (DR) estimator developed by \citet{Jiang16} and \citet{Thomas16} in that we use the cross-fitting procedure, as explained next. 

{\bf A. Cross-Fitting.}
The above method uses a sample-splitting procedure called {\it cross-fitting}, where we split the data into $K$ folds, take the sample analogue of Eq. (\ref{eq:value}) using one of the folds ($I_k$) with $\pi_b$ and $\{ q_t^{\pi_e}\}_{t=0}^T$ estimated from the remaining folds ($I_k^c$) plugged in, and average the estimates over the $K$ folds to produce a single estimator.
Cross-fitting has two advantages.
First, if we use instead the whole sample both for estimating $\pi_b$ and $\{ q_t^{\pi_e}\}_{t=0}^T$ and for computing the final estimate of $V^{\pi_e}$ (the ``full-data'' variant of the DR estimator of \citet{Thomas16}), substantial bias might arise due to overfitting \cite{Chernozhukov2018,newey2018crossfitting}.
Cross-fitting removes the potential bias by making the estimate of $V^{\pi_e}$ independent of the estimates of $\pi_b$ and $\{ q_t^{\pi_e}\}_{t=0}^T$.

Second, a standard sample splitting procedure uses a half of the data to construct estimates of $\pi_b$ and $\{ q_t^{\pi_e}\}_{t=0}^T$ and the other half to compute the estimate of $V^{\pi_e}$ (the DR estimator of \citet{Jiang16} and the ``half-data'' variant of the DR estimator of \citet{Thomas16}). 
In contrast, cross-fitting swaps the roles of the main fold ($I_k$) and the rest ($I_k^c$) so that all trajectories are used for the final estimate, which enables us to make efficient use of data.

{\bf B. Neyman Orthogonality.}
There is another key ingredient important for DML to have desirable properties.
The DML estimator is constructed by plugging in the estimates of $\pi_b$ and $\{ q_t^{\pi_e}\}_{t=0}^T$, which may be biased due to regularization if they are estimated with machine learning methods.
However, the DML estimator is robust to the bias, since $\psi$ satisfies the {\it Neyman orthogonality} condition \cite{Chernozhukov2018}. 
The condition requires that for any candidate tuple $\tilde \eta=(\tilde \pi_b,\{\tilde q_t^{\pi_e}\}_{t=0}^T)$,
$$
\left. \frac{\partial \mathbb{E}_{H\sim \pi_b}[\psi(H;\eta+r(\tilde \eta-\eta))]}{\partial r}\right|_{r=0}=0,
$$
where $\eta=(\pi_b,\{q_t^{\pi_e}\}_{t=0}^T)$ is the tuple of the true functions (see Appendix \ref{proof} for the proof that DML satisfies this). 
Intuitively, the Neyman orthogonality condition means that the right-hand side of Eq. (\ref{eq:value}) is locally insensitive to the value of $\pi_b$ and $\{q_t^{\pi_e}\}_{t=0}^T$.
More formally, the first-order approximation of the bias caused by using $\tilde \eta$ instead of the true $\eta$ is given by
\begin{align*}
\mathbb{E}_{H\sim \pi_b}[\psi(H;\tilde \eta))] -\mathbb{E}_{H\sim \pi_b}[\psi(H;\eta)]
\approx\left. \frac{\partial \mathbb{E}_{H\sim \pi_b}[\psi(H;\eta+r(\tilde \eta-\eta))]}{\partial r}\right|_{r=0},
\end{align*}
which is exactly zero by the Neyman orthogonality condition.
As a result, plugging in noisy estimates of $\pi_b$ and $\{q_t^{\pi_e}\}_{t=0}^T$ does not strongly affect the estimate of $V^{\pi_e}$.

In contrast, IPW is based on the following expression of $V^{\pi_e}$: $V^{\pi_e}=\mathbb{E}_{H\sim \pi_b}[\psi_{{\rm IPW}}(H;\pi_b)]$, where for any candidate $\tilde \pi_b$, we define
\begin{align}
\psi_{{\rm IPW}}(H;\tilde \pi_b)=\sum_{t=0}^T \gamma^t \rho_t^{\pi_e}(H_{t}^{s,a};\tilde \pi_b)R_{t}. \label{eq:IPW}
\end{align}
The Neyman orthogonality condition does not hold for IPW: for some $\tilde \pi_b\neq \pi_b$,
$
\left. \frac{\partial \mathbb{E}_{H\sim \pi_b}[\psi_{{\rm IPW}}(H;\pi_b+r(\tilde \pi_b-\pi_b))]}{\partial r}\right|_{r=0}\neq0.
$
Therefore, IPW is not robust to the bias in the estimate of $\pi_b$. 

\subsection{$\sqrt{N}$-consistency and Asymptotic Normality}
Let $
\sigma^2 =\mathbb{E}_{H\sim \pi_b}[(\psi(H;\eta)-V^{\pi_e})^2]$ 
be the variance of $\psi(H;\eta)$.
To derive the properties of the DML estimator, we make the following assumption.

\begin{assumption}\label{as:regularity}
	\begin{enumerate}
		\renewcommand{\theenumi}{(\alph{enumi})}
		\renewcommand{\labelenumi}{(\alph{enumi})}
		\item \label{as:variance} There exists a constant $c_0>0$ such that $c_0\le \sigma^2 <\infty$ for all $N$.
		\item \label{as:nuisance-est} For each $k=1,...,K$, the estimator $\hat \eta_k=(\hat \pi_{b,k},\{\hat q_{t,k}^{\pi_e}\}_{t=0}^T)$ belongs to a set ${\mathcal T}_N$ with probability approaching one, where ${\mathcal T}_N$ contains the true $\eta=(\pi_b,\{q_{t}^{\pi_e}\}_{t=0}^T)$ and satisfies the following:
		\begin{enumerate}[(i)]
			\item \label{as:moment} There exist constants $q>2$ and $c_1>0$ such that
			$$\sup_{\tilde \eta\in {\mathcal T}_N}\left(\mathbb{E}_{H\sim \pi_b}[(\psi(H;\tilde\eta)-V^{\pi_e})^q]\right)^{1/q}\le c_1$$
			for all $N$.
			\item \label{as:dif}
			$
			\sup_{\tilde \eta\in {\mathcal T}_N}\left(\mathbb{E}_{H\sim \pi_b}[(\psi(H;\tilde \eta)-\psi(H;\eta))^2]\right)^{1/2} = o(1).
			$
			\item \label{as:second_d}
			$
			\sup_{r\in(0,1),\tilde \eta\in {\mathcal T}_N} \left|\tfrac{\partial^2 \mathbb{E}_{H\sim \pi_b}[\psi(H;\eta+r(\tilde \eta-\eta))]}{\partial r^2}\right|=o(1/\sqrt{N}).
			$
		\end{enumerate}
	\end{enumerate}
\end{assumption}

Assumption \ref{as:regularity} \ref{as:variance} assumes that the variance of $\psi(H;\eta)$ is nonzero and finite.
Assumption \ref{as:regularity} \ref{as:nuisance-est} states that the estimator $(\hat \pi_{b,k},\{\hat q_{t,k}^{\pi_e}\}_{t=0}^T)$ belongs to the set ${\mathcal T}_N$, a shrinking neighborhood of the true $(\pi_b,\{q_{t}^{\pi_e}\}_{t=0}^T)$, with probability approaching one.
It requires that $(\hat \pi_{b,k},\{\hat q_{t,k}^{\pi_e}\}_{t=0}^T)$ converge to $(\pi_b,\{q_{t}^{\pi_e}\}_{t=0}^T)$ at a sufficiently fast rate so that the rate conditions in Assumption \ref{as:regularity} \ref{as:nuisance-est} are satisfied.

The following proposition establishes the $\sqrt{N}$-consistency and asymptotic normality of $\hat V^{\pi_e}_{{\rm DML}}$ and provides a consistent estimator for the asymptotic variance.
Below, $\rightsquigarrow$ denotes convergence in distribution, and $\stackrel{p}{\rightarrow}$ denotes convergence in probability.
\begin{proposition}\label{prop:step-DR}
	If Assumption \ref{as:regularity} holds, then
	\begin{align*}
	\sqrt{N}\sigma^{-1}(\hat V^{\pi_e}_{{\rm DML}}-V^{\pi_e})\rightsquigarrow N(0,1) \text{ and }
	\hat \sigma^2-\sigma^2 \stackrel{p}{\rightarrow} 0,
	\end{align*}
	where $\hat \sigma^2=\frac{1}{K}\sum_{k=1}^K \frac{1}{n}\sum_{i\in I_k}(\psi(H_i;\hat \eta_k)-\hat V^{\pi_e}_{{\rm DML}})^2$ is a sample analogue of $\sigma^2$.
\end{proposition}
The proof is an application of Theorems 3.1 and 3.2 of \citet{Chernozhukov2018}, found in Appendix \ref{proof}.
The above convergence result holds under any sequence of probability distributions $\{\mathbb{P}_N\}_{N\ge 1}$ as long as Assumption \ref{as:regularity} holds.
Therefore, our approach is usable, for example, in the case where there are a growing number of possible states, that is, $|{\mathcal S}|$ is increasing with $N$.

\begin{remark}
	It is possible to show that, if we impose an additional condition about the class of estimators of $\pi_b$ and $\{q_t^{\pi_e}\}_{t=0}^T$, the ``full-data'' variant of the DR estimator (the one that uses the whole sample both for estimating $\pi_b$ and $\{ q_t^{\pi_e}\}_{t=0}^T$ and for computing the estimate of $V^{\pi_e}$) has the same convergence properties as our estimator \cite[Theorems 6 \& 8]{kallus2019DRL}.
	Cross-fitting enables us to prove the desirable properties under milder conditions than those necessary without sample splitting.
\end{remark}

\begin{remark}
	For the ``half-data'' variant of the DR estimator (the one that uses a half of the data to estimate $\pi_b$ and $\{ q_t^{\pi_e}\}_{t=0}^T$ and the other half to compute the estimate of $V^{\pi_e}$), 
	a version of Proposition 1 in which $\sqrt{N}$ is replaced with $\sqrt{N/2}$ holds, since this method only uses a half of the data to construct the final estimate.
	This reduction in the data size leads to the variance of this estimator being roughly twice as large as that of the DML estimator, which formalizes the efficiency of our method.
\end{remark}

\subsection{Contextual Bandits as A Special Case}
When $T=0$, a trajectory takes the form of $H=(S_0,A_0,R_0)$.
Regarding $S_0$ as a context, it is possible to consider $\{H_i\}_{i=1}^N$ as batch data generated by a contextual bandit algorithm.
In this case, the DML estimator becomes
\begin{align*}
\hat V^{\pi_e}_{{\rm DML}} 
=\frac{1}{K}\sum_{k=1}^K \frac{1}{n}\sum_{i\in I_k} &\Bigl\{\frac{\pi_e(A_{i0}|S_{i0})}{\hat\pi_{b,k}(A_{i0}|S_{i0})} (R_{i0}-\hat\mu_k(S_{i0},A_{i0}))\Bigr.+\Bigl.\sum_{a\in\mathcal{A}}\pi_e(a|S_{i0})\hat\mu_k(S_{i0},a)\Bigr\},
\end{align*}
where $(\hat \pi_{b,k},\hat \mu_k)$ is the estimator of $(\pi_b,\mu)$ using the subset of data $\{H_i\}_{i\in I_k^c}$.
This estimator is the same as the DR estimator of \citet{Dudik2014} except that we use the cross fitting procedure.
Proposition 1 has the following implication for the contextual bandit case.
Let $\sigma^2_R(s,a)=\int (r-\mu(s,a))^2dP_R(r|s,a)$.
\begin{corollary}
	Suppose that $T=0$ and that Assumption 1 holds.
	Then,
	\begin{align*}
	    \sqrt{N}\sigma^{-1}_{CB}(\hat V^{\pi_e}_{{\rm DML}}-V^{\pi_e})\rightsquigarrow N(0,1),
	\end{align*}
	where
	\begin{align*}
	    \sigma^2_{CB}=\mathbb{E}_{S_0\sim P_{S_0}}&\Bigl[\sum_{a\in\mathcal{A}}\frac{\pi_e(a|S_0)^2}{\pi_b(a|S_0)}\sigma^2_R(S_0,a)\Bigr.+\Bigl.\Bigl(\sum_{a\in\mathcal{A}}\pi_e(a|S_0)\mu(S_0,a)-V^{\pi_e}\Bigr)^2\Bigr].
	\end{align*}
\end{corollary}
The variance expression coincides with the ``semiparametric efficiency bound'' obtained by \citet{narita2018efficient}, where 
the semiparametric efficiency bound is the smallest possible asymptotic variance among all consistent and asymptotically normal estimators. 
Hence $\hat V^\pi_{{\rm DML}}$ is the lowest variance estimator.

\section{Experiments}

\subsection{Simulation Experiment: CartPole-v0}
We demonstrate the effectiveness of our proposed estimator in the OpenAI Gym CartPole-v0 environment \cite{brockman2016openai}.
In this environment, the agent decides between two actions, moving the cart left or right, so that the pole attached to it stays upright.
We make the following two changes to this environment to make it a more high-dimensional setting.
First, the inputs to the agent in the original environment are 4 values representing the state (cart position, cart velocity, pole angle, and pole velocity at the tip), but we instead use a patch of the screen centered on the cart as an input, using Reinforcement Learning (DQN) Tutorial \cite{NEURIPS2019_9015}.\footnote{\url{https://pytorch.org/tutorials/intermediate/reinforcement_q_learning.html}}
We then represent the state ($S_t$) as the difference between the current screen patch and the previous one.
This leads to an environment with a larger number of state variables.
Second, at the beginning of every episode, we randomly select $10\%$ of pixels in the screen patch and black them out. We then move the pixels painted black horizontally over time.
In this environment, some pixels that have nothing to do with the movement of the cart and pole become black or white during the episode, and the state variables representing these pixels are completely irrelevant to the reward.
This mimics a situation where we do not know which among a large number of state variables are actually important, which makes the environment suitable for testing the performance of our method.

We use a greedy policy from a learned Q function as the evaluation policy ($\pi_e$) and an $\epsilon$-greedy policy as the behavior policy ($\pi_b$). 
Following Reinforcement Learning (DQN) Tutorial \cite{NEURIPS2019_9015}, we use a convolutional neural network to learn the Q function. 
Our convolutional network has the following structure: a convolution layer with 3 $16 \times 16$ filters and stride 2, followed by a $16 \times 16$ batch normalization layer, followed by a convolution layer with 16 $32 \times 32$ filters and stride 2, followed by a $32 \times 32$ batch normalization layer, followed by a convolution layer with 32 $32 \times 32$ filters and stride 2, followed by a $32 \times 32$ batch normalization layer, followed by a dense layer, and finally the binary output layer. All hidden layers employ ReLU activation functions.

Our experiment proceeds as follows. Firstly, we run the behavior policy to collect 100 trajectories and use them to estimate the expected value of the evaluation policy. Secondly, we run the evaluation policy to observe the ground truth value. We iterate this procedure 100 times and then report the mean squared error (MSE) as a performance metric.


We compare DML (our proposal) with four baseline methods: Direct Method (DM), also known as Regression Estimator, IPW, full-data DR, and half-data DR.
DM, IPW and full-data DR are given by
\begin{align*}
\hat V_{{\rm DM}}^{\pi_e}&=\frac{1}{N}\sum_{i=1}^N \sum_{a\in{\mathcal A}}\pi_e(a|S_{i0})\hat q_0^{\pi_e}(S_{i0},a),\\
\hat V_{{\rm IPW}}^{\pi_e}&=\frac{1}{N}\sum_{i=1}^N \psi_{{\rm IPW}}(H_i;\hat\pi_b),\\
\hat V_{{\rm DR-full}}^{\pi_e}&=\frac{1}{N}\sum_{i=1}^N \psi(H_i;\hat\pi_b, \{\hat q_t^{\pi_e}\}_{t=0}^T),
\end{align*}
where $\psi$ and $\psi_{{\rm IPW}}$ are defined in Eq. (\ref{eq:DR}) and Eq. (\ref{eq:IPW}), respectively, and $(\hat \pi_{b},\{\hat q_t^{\pi_e}\}_{t=0}^T)$ is the estimator of $(\pi_b,\{q_t^{\pi_e}\}_{t=0}^T)$ using the {\it whole} data.
Half-data DR is given by
$$
\hat V_{{\rm DR-half}}^{\pi_e}=\frac{1}{(N/2)}\sum_{i\in I_1} \psi(H_i;\hat\pi_{b,1}, \{\hat q_{b,1}^{\pi_e}\}_{t=0}^T),
$$
where $I_1$ is a random subset of trajectory indices $\{1,...,N\}$ of size $N/2$, and $(\hat \pi_{b,1},\{\hat q_{t,1}^{\pi_e}\}_{t=0}^T)$ is the estimator of $(\pi_b,\{q_t^{\pi_e}\}_{t=0}^T)$ using the other half of the data.
$\hat V_{{\rm DM}}^{\pi_e}$ and $\hat V^{\pi_e}_{{\rm IPW}}$ do not satisfy the Neyman orthogonality condition nor use cross-fitting. 
$\hat V^{\pi_e}_{{\rm DR-full}}$ and $\hat V^{\pi_e}_{{\rm DR-half}}$ satisfy the Neyman orthogonality condition but do not use cross-fitting. 
DML is therefore expected to outperform these four baseline methods.

We use a convolutional neural network to estimate $\pi_b$ and $\{q_t^{\pi_e}\}_{t=0}^T$.
We use the same model structure as the one used to construct the evaluation policy and behavior policy.
We set the learning rate to $0.0001$ for estimating the Q function and to $0.0002$ for estimating the behavior policy. The batch size is $128$ for both models.  

Note that for DML, we hold out a fold and use the rest of the data to obtain the estimators of $\pi_b$ and $\{q_t^{\pi_e}\}_{t=0}^T$. We use $K = 2$ folds.
We verify that increasing $K$ to 3 or 4 changes the results little.


\begin{table*}[t]\centering
		\caption{Comparing Off-Policy Evaluations: CartPole-v0}
		\label{noise-dqn-result}
		
		\begin{tabular}{l|rrrrrr}
			\hline
			\multicolumn{1}{l|}{} & \multicolumn{6}{c}{MSE}\\
			Methods & $\epsilon = 0.05$ & $\epsilon = 0.06$ & $\epsilon = 0.07$ & $\epsilon = 0.08$ & $\epsilon = 0.09$ & $\epsilon = 0.1$ \\ 
			\hline
			DM & 19.60 (1.96) & 23.32 (2.26) & 22.32 (2.49) & 20.51 (1.71) & 25.64 (2.22) & 28.01 (2.54) \\ 
			IPW & 300.71 (90.92) & 201.17 (66.55) & 215.15 (53.47) & 131.44 (47.14) & 170.31 (53.02) & 100.93 (33.2) \\ 
			DR-full & 32.03 (6.63) & 27.43 (6.28) & 23.25 (4.09) & 16.00 (2.47) & 14.44 (2.43) & 12.63 (1.98) \\ 
			DR-half & 8.7(1.31) & 12.18(1.72) & 11.84(1.38) & 9.5(1.15) & 12.53(1.35) & 13.47(1.73) \\ 
			DML & {\bf 5.31} (0.7) & {\bf 6.82} (0.74) & {\bf 6.47} (0.83) & {\bf 7.07} (0.79) & {\bf 9.69} (1.01) & {\bf 10.52} (1.22) \\ 			
			\hline
		\end{tabular}
\begin{flushleft}
		\fontsize{8.0pt}{10.0pt}\selectfont \textit{Notes}: This table shows MSEs of the predicted rewards of the evaluation policy compared to its actual reward. The standard errors of these MSEs are in parentheses. $\epsilon$ is the probablity that the $\epsilon$-greedy algorithm chooses an action randomly.
\end{flushleft}
\end{table*}

The MSEs of the five estimators are reported in Table \ref{noise-dqn-result}.
Regardless of the value of $\epsilon$ (the probability that the $\epsilon$-greedy algorithm chooses an action randomly), DML outperforms the other estimators.
Since the CartPole-v0 environment tends to have long trajectories, the importance weight (Eq. (\ref{eq:weight})) tends to be very volatile due to the curse of the horizon.
As a result, the MSE of IPW is much larger than the other four.
When $\epsilon$ is smaller, the performance of full-data DR also seems to be severely affected by the importance weight and is worse than DM.
As $\epsilon$ increases, this effect gets weaker, and full-data DR becomes better than DM.
Half-data DR performs much better than DM and IPW for any $\epsilon$ thanks to Neyman Orthogonality and sample splitting, but performs worse than DML due to the lack of cross-fitting.
This simulation result thus highlights the validity of the DML estimator even in the challenging situation where there are a large number of potentially important state variables.

\subsection{Real-World Experiment: Online Ads}

We apply our estimator to empirically evaluate the design of online advertisements.
This application uses proprietary data provided by CyberAgent Inc., which we described in the introduction.
This company uses bandit algorithms to determine the visual design of advertisements assigned to user impressions in a mobile game. 

Our data are logged data from a 7-day A/B test on ad ``campaigns,'' where each campaign randomly uses either a multi-armed bandit (MAB) algorithm or a contextual bandit (CB) algorithm for each user impression.
The parameters of both algorithms are updated every 6 hours.
As the A/B test lasts for 7 days, there are 28 batches in total. 
The parameters of both algorithms stay constant across impressions within each batch.

In the notation of our theoretical framework, $T$ is zero, and a trajectory takes the form of $H=(S_0,A_0,R_0)$.
Reward $R_0$ is a click while action $A_0$ is an advertisement design. The number of possible designs (the size of the action space) is around 55 and slightly differs depending on the campaign.
State (or context) $S_0$ is user and ad characteristics used by the algorithm.
$S_0$ is high dimensional and has tens of thousands of possible values. 

We utilize this data to examine the performance of our method.
For each campaign and batch, we regard the MAB algorithm as the behavior policy $\pi_b$ and the CB algorithm as the evaluation policy $\pi_e$.
We use the log data from users assigned to the MAB algorithm to estimate the expected value of the CB algorithm.
We also compute the actual click rate using the log data from users assigned to the CB algorithm, which serves as the ground truth.
We then compare the predicted value with the actual one. 

We consider the relative root mean squared error (relative-RMSE) as a performance metric.
Let $C$ and $B$ denote the number of campaigns ($C=4$) and the number of batches ($B=28$), respectively. Let $N_{c,b}$ denote the number of user impressions assigned to the CB algorithm for campaign $c$ in batch $b$. Let $\hat V^{\pi_e}_{c,b}$ and $\bar V^{\pi_e}_{c,b}$ be the estimated value and the actual value (the click rate) for campaign $c$ in batch $b$. 
We define relative-RMSE as follows:
\begin{align*}
\text{relative-RMSE}
&=\left(\tfrac{1}{\sum_{c=1}^C\sum_{b=1}^{B} N_{c,b}}\sum_{c=1}^C\sum_{b=1}^{B} N_{c,b}\left(\tfrac{\hat V^{\pi_e}_{c,b}-\bar V^{\pi_e}_{c,b}}{\bar V^{\pi_e}_{c,b}}\right)^2\right)^{1/2}.
\end{align*}
As the actual click rate $\bar V^{\pi_e}_{c,b}$ varies across campaigns and batches, we normalize the prediction error $\hat V^{\pi_e}_{c,b}-\bar V^{\pi_e}_{c,b}$ by dividing it by the actual value $\bar V^{\pi_e}_{c,b}$ to equally weight every campaign-batch combination.

\begin{table*}[t]\centering
		\caption{Comparing Off-Policy Evaluations: Online Ads}\label{ca-bandit-result}
		\begin{tabular}{l|c|c}
			\hline
			\multicolumn{1}{l|}{} & \multicolumn{1}{c|}{Estimated Behavior Policy} & \multicolumn{1}{c}{True Behavior Policy} \\
			Method & relative-RMSE & relative-RMSE  \\ 
			\hline
			IPW & 0.72531 (0.03077)  & 0.72907 (0.03061) \\ 
			DR-full  & 0.65470 (0.02900) & 0.65056 (0.02880)  \\ 
			DR-half  & 0.66331 (0.02726)& 0.67209 (0.02527)  \\ 
			DML  & {\bf 0.56623} (0.02952)  & {\bf 0.56196} (0.02922)  \\ 
			\hline 
			Sample size & \multicolumn{2}{c}{\# impressions assigned to the MAB algorithm = 40,101,050} \\ & \multicolumn{2}{c}{\# impressions assigned to the CB algorithm = 299,342} \\
			\hline 
		\end{tabular}
\begin{flushleft}
		\fontsize{8.0pt}{10.0pt}\selectfont \textit{Notes}: This table shows relative-RMSEs of the predicted CTRs of the evaluation policy compared to its actual CTR. The standard errors of these RMSEs are in parentheses.
\end{flushleft}
\end{table*}

We examine the performance of four estimators, IPW \cite{Strehl2010}, full-data DR, half-data DR, and DML (our proposal).
We use LightGBM \cite{ke2017lightgbm} to estimate the mean reward function $\mu$.
First, we split the data used for estimating $\mu$ into training and validation sets, train the reward model with the training set, and tune hyperparameters with the validation set.
We then use the tuned model and the original data to obtain the reward estimator.
We estimate the behavior policy $\pi_b$ by taking the empirical shares of ad designs in the log data.
We also try using the true behavior policy, computed by the Monte Carlo simulation from the beta distribution used in Thompson Sampling the MAB algorithm uses.
For DML, we use $K = 2$ folds to perform cross-fitting. We verify that increasing $K$ to 3 or 4 changes the results little.

We present our key empirical result in Table \ref{ca-bandit-result}. 
Regardless of whether the true behavior policy is available or not, DML outperforms IPW by more than 20 \% and outperforms both of the two variants of DR by more than 10 \% in terms of relative-RMSE.
These differences are statistically significant at 5\% level. This result supports the practical value of our proposed estimator. 

Finally, we use our method to measure the performance of a new counterfactual policy of choosing the ad design predicted to maximize the CTR. We obtain the counterfactual policy by training a click prediction model with LightGBM on the data from the first $B-1$ batches. In the training, we set the number of leaves to $20$ and the learning rate to $0.01$, and decide the number of boost rounds by cross-validation. We then use our proposed estimator to predict the performance of the counterfactual policy on the data from the last batch ($b=B$). The last batch contains three campaigns. The resulting off-policy evaluation results are reported in Figure \ref{table2}. The results show that the counterfactual policy performs better than the existing algorithms in two of the three campaigns, with a statistically significant improvement in one of the campaigns.

\begin{figure}[t]
	\includegraphics[width=8cm]{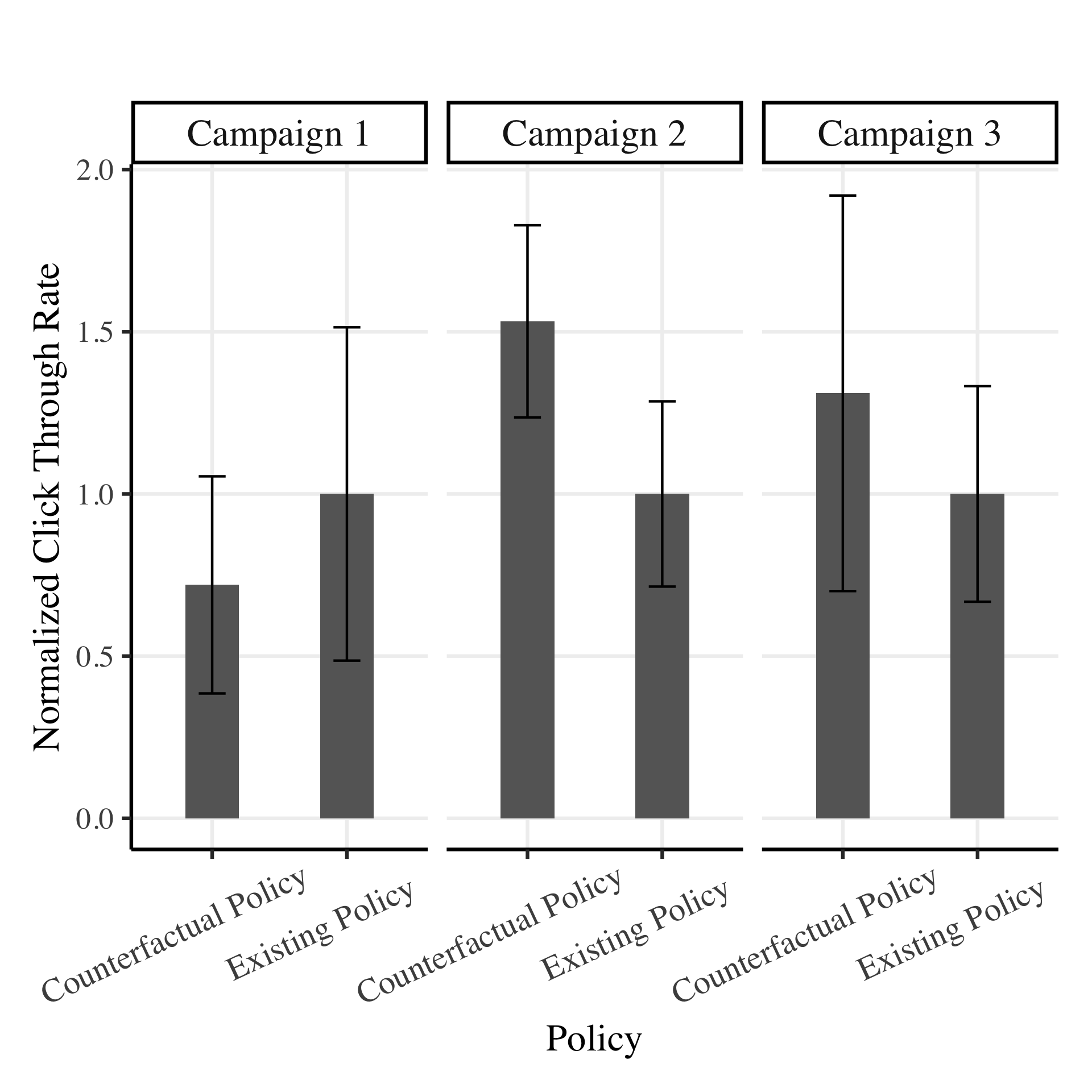}
	\caption{Improving Ad Design}\label{table2}
	\begin{flushleft}
		\fontsize{8.0pt}{10.0pt}\selectfont \textit{Notes}: This figure shows estimates of the expected CTRs of the logging policy and a counterfactual policy of choosing the best action for each context. Three panels correspond to three campaigns. CTRs are multiplied by a constant for confidentiality reasons. We obtain these estimates by our proposed estimator (the DML estimator using the estimated behavior policy). Bars indicate 95\% confidence intervals based on our asymptotic variance estimators developed in Proposition \ref{prop:step-DR}.
\end{flushleft}
\end{figure}

\section{Related Work}
In this section, we discuss the relationships between several most closely related papers and ours.
Our methodological building block is the \citet{Chernozhukov2018}'s DML framework for estimation and inference of semiparametric models.
Our paper is among the first to show the applicability of their framework to the problem of off-policy evaluation for RL and bandits.
In addition, we validate the practical importance of the two features of DML, cross-fitting and Neyman orthogonality, by both simulation and real world experiments.

Within the off-policy evaluation literature, our paper mostly builds on the work by \citet{Jiang16} and \citet{thomas2016data}, who propose DR estimators.
\citet{Jiang16} are the first to apply the DR technique to off-policy evaluation for RL.
\citet{thomas2016data} extend \citet{Jiang16}'s method to the settings where the time horizon is infinite and/or unknown, and propose a new estimator which combines a model-based estimator with a DR estimator.
As explained in Section 3, the DR estimators of these two papers and our DML are based on the same expression of $V^{\pi_e}$ (Eq. (\ref{eq:value})).
The difference between our estimator and theirs is that we use the cross-fitting procedure while they either use the full sample both for estimating $\pi_b$ and $\{ q_t^{\pi_e}\}_{t=0}^T$ and for computing the final estimate of $V^{\pi_e}$, or use a half of the data to construct estimates of $\pi_b$ and $\{ q_t^{\pi_e}\}_{t=0}^T$ and the other half to compute the estimate of $V^{\pi_e}$.
In their experiments, \citet{Jiang16} also implement cross-fitting as one variant of the DR estimator, which is the same as our DML estimator if we use the true behavior policy $\pi_b$.
However, we (i) allow the behavior policy to be unknown, (ii) allow for having a large number of potentially important state variables, and (iii) present statistical properties of the estimator while they do not.

The most closely related paper to ours is \citet{kallus2019DRL}, who also apply DML to off-policy evaluation for RL. 
Their work and our work are independent and simultaneous.
More substantively, there are several key differences between their paper and ours.
Empirically, their paper does not have an application to a real dataset.
In contrast, we show the estimator's practical performance in a real product setting. 
In addition, we provide a consistent estimator of the asymptotic variance, and allow for having a large number of potentially important state variables while they do not.
Methodologically, their estimator differs from ours in that they estimate the marginal distribution ratio $\frac{\Pr_{H\sim \pi_e}(S_t,A_t)}{\Pr_{H\sim \pi_b}(S_t,A_t)}$ and plug it in place of $\rho_t^{\pi_e}(H_t^{s,a};\pi_b)$ of our estimator.
There are tradeoffs between the two approaches.
Their estimator is better in that it has a smaller asymptotic variance.
In the important case with a known behavior policy, however, our estimator has the following advantages.
(1) Our approach only needs to estimate $\{q_t^{\pi_e}\}_{t=0}^T$, while their approach needs to estimate the marginal distribution ratio as well as $\{q_t^{\pi_e}\}_{t=0}^T$.
(2) Our estimator is $\sqrt{N}$-consistent and can be asymptotic normal even when the nonparametric estimator of $\{q_t^{\pi_e}\}_{t=0}^T$ does not converge to the true one, or when the parametric model of $\{q_t^{\pi_e}\}_{t=0}^T$ is misspecified.
In contrast, their estimator may not be even consistent in such a case unless the nonparametric estimator of the marginal distribution ratio is consistent, or the parametric model of the marginal distribution ratio is correctly specified.


\section{Conclusion} 

This paper proposes a new off-policy evaluation method, by marrying the DR estimator with Double/Debiased Machine Learning.
Our estimator has two features.
First, unlike the IPW estimator, it is robust to the bias in the estimates of the behavior policy and of the state-action value function (Neyman orthogonality).
Second, we use a sample-splitting procedure called cross-fitting.
This removes the overfitting bias that would arise without sample splitting but still makes full use of data, which makes our estimator better than DR estimators.
Theoretically, we show that our estimator is $\sqrt{N}$-consistent and asymptotically normal with a consistent variance estimator, thus allowing for correct statistical inference.
Our experiments show that our estimator outperforms the standard DR, IPW and DM estimators in terms of the mean squared error.
This result not only demonstrates the capability of our estimator to reduce prediction errors, but also suggests the more general possibility that 
the two features of DML (Neyman orthogonality and cross-fitting) may improve many variants of the DR estimator such as MAGIC \cite{Thomas16}, SWITCH \cite{wang2016optimal} and MRDR \cite{Farajtabar2018MoreRD}.

\bibliographystyle{ACM-Reference-Format}
\bibliography{reference}

\newpage

\section*{Appendices}

\appendix

\section{Examples}

The data generating process in our framework allows for many popular RL and bandit algorithms, as the following examples illustrate. 

\begin{example}[Deep $Q$ Learning]\label{ex:deepQ}
	In each round $t$, given state $s_t$, a Q Learning algorithm picks the best action based on the estimated Q-value of each actions, $Q(s, a)$, which estimates the expected cumulative reward from taking action $a$ (following the state and the policy). Choice probabilities can be determined with an $\epsilon$-greedy or soft-max rule, for instance. In the case where the soft-max rule is employed, the probability of taking each action is as follows:
	$$
	\pi(a|s_t) = \frac{ \exp ( Q(s_t,a)) }{\sum_{a' \in A} \exp ( Q(s_t,a')) }.
	$$
	Deep Q Learning algorithms estimate Q-value functions through deep learning methods.
	
\end{example}

\begin{example}[Actor Critic]\label{ex:AC}
	An Actor Critic is a hybrid method of value-based approach such as Q learning and policy-based method such as REINFORCE. This algorithm has two components called Actor and Critic. Critic estimates the value function and Actor updates the policy using the value of Critic. In each round t, we pick the best action according to the value of Actor with some probability. As in Deep $Q$ Learning algorithms, we can use $\epsilon$-greedy and soft-max for determining an action.
\end{example}

Contextual bandit algorithms are also important examples. 
When $T=0$, a trajectory takes the form of $H=(S_0,A_0,R_0)$.
Regarding $S_0$ as a context, it is possible to consider $\{H_i\}_{i=1}^N$ as batch data generated by a contextual bandit algorithm.
In additional examples below, the algorithms use past data to estimate the mean reward function $\mu$ and the reward variance function $\sigma^2_R$, where $\sigma^2_R(s,a)=\int (r-\mu(s,a))^2dP_R(r|s,a)$.
Let $\hat\mu$ and $\hat \sigma^2_R$ denote any given estimators of $\mu$ and $\sigma^2_R$, respectively.

\begin{example}[$\epsilon$-greedy]\label{ex:e-greedy}
	When the context is $s_0$, we choose the best action based on $\hat \mu(s_0,a)$ with probability $1-\epsilon$ and choose an action uniformly at random with probability $\epsilon$:
	$$
	\pi(a|s_0)= \begin{cases}
	1-\epsilon + \frac{\epsilon}{|\mathcal{A}|}& \ \ \ \text{if $a=\argmax_{a'\in \mathcal{A}}\hat \mu(s_0,a')$}\\
	\frac{\epsilon}{|\mathcal{A}|} & \ \ \ \text{otherwise}.
	\end{cases}
	$$
\end{example}

\begin{example}[Thompson Sampling using Gaussian priors]\label{ex:Thompson}
	When the context is $s_0$, we sample the potential reward $r_0(a)$ from the normal distribution $N(\hat \mu(s_0,a),\hat \sigma^2_R(s_0,a))$ for each action, and choose the action with the highest sampled potential reward, $\argmax_{a'\in \mathcal{A}}r_0(a')$.
	As a result, this algorithm chooses actions with the following probabilities:
	$$
	\pi(a|s_0) = \Pr(a=\argmax_{a'\in \mathcal{A}}r_0(a')),
	$$
	where $(r_0(a))_{a\in\mathcal{A}}\sim N(\hat \mu(s_0),\hat \Sigma(s_0))$, $\hat \mu(s_0)=(\hat \mu(s_0,a))_{a\in\mathcal{A}}$, and
	$\hat \Sigma(s_0)$ is the diagonal matrix whose diagonal entries are $(\sigma^2_R(s_0,a))_{a\in\mathcal{A}}$.
\end{example}

\section{Additional Figure}

We show some examples of visual ad designs in our real product application in Figure \ref{fig:ads_sample}. 

\begin{figure*}
	\begin{center}
		\includegraphics[width=0.8\linewidth]{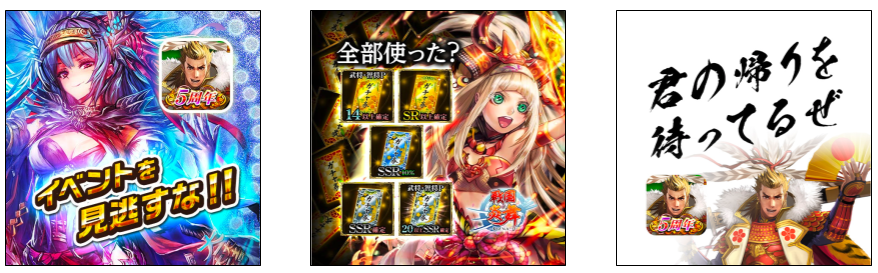}
		\caption{Examples of Visual Ad Designs}\label{fig:ads_sample}
	\end{center}
\end{figure*}

\section{Standard Error Calculations in the Online Ad Experiment}

We calculate the standard error (statistical uncertainty) of relative-RMSE by a bootstrap-like procedure.
This procedure is based on normal approximation of the distributions of $\hat V^{\pi_e}_{c,b}$ and $\bar V^{\pi_e}_{c,b}$: for each $(c,b)$, $\hat V^{\pi_e}_{c,b}\sim N(V^{\pi_e}_{c,b},\frac{\hat \sigma^{2,ope}_{c,b}}{N^{ope}_{c,b}})$ and $\bar V^{\pi_e}_{c,b}\sim N(V^{\pi_e}_{c,b},\frac{\hat \sigma^{2,online}_{c,b}}{N_{c,b}})$, where $V_{c,b}^{\pi_e}$ is the true value of policy $\pi_e$, $\hat \sigma^{2,ope}_{c,b}$ is the estimator for the asymptotic variance of $\hat V^{\pi_e}_{c,b}$ given in Proposition 1, $N^{ope}_{c,b}$ is the number of impressions used to estimate $\hat V^{\pi_e}_{c,b}$, and $\hat \sigma^{2,online}_{c,b}=\bar V^{\pi_e}_{c,b}(1-\bar V^{\pi_e}_{c,b})$ is the sample variance of the click indicator among the impressions assigned to the CB algorithm.

The standard error is computed as follows.
First, we compute $\hat\sigma^{2,ope}_{c,b}$ and $\hat\sigma^{2,online}_{c,b}$ for each $(c,b)$.
Second, we draw $\hat V^{\pi_e,sim}_{c,b}$ and $\bar V^{\pi_e,sim}_{c,b}$ independently from $N(\hat V^{\pi_e}_{c,b},\frac{\hat \sigma^{2,ope}_{c,b}}{N^{ope}_{c,b}})$ and $N(\hat V^{\pi_e}_{c,b},\frac{\hat \sigma^{2,online}_{c,b}}{N_{c,b}})$ for every $(c,b)$, and calculate the relative-RMSE using the draws $\{(\hat V^{\pi_e,sim}_{c,b},\bar V^{\pi_e,sim}_{c,b})\}_{c=1,...,C, b=1,...,B}$.
We then repeat the second step 100,000 times, and compute the standard deviation of the simulated relative-RMSEs.

\section{Lemmas}\label{lemma}


\begin{lemma}\label{lemma:step-IS:R}
	For $t=0,...,T$, $\mathbb{E}_{H\sim \pi_b}[\rho_t^{\pi_e}(H_t^{s,a})R_t]=\mathbb{E}_{H\sim \pi_e}[R_t]$.
\end{lemma}
\begin{proof}
	Let $P^{\pi}(h_t^{s,a})$ denote the probability of observing trajectory $h_t^{s,a}=(s_0,a_0,...,s_t,a_t)$ when $H\sim \pi$ for some policy $\pi$.
	Under our data generating process, 
	$$
	P^{\pi}(h_0^{s,a})=P_{S_0}(s_0)\pi(a_0|s_0),
	$$
	and for $t\ge 1$,
	\begin{align*}
	P^{\pi}(h_t^{s,a})=&P_{S_0}(s_0)\pi(a_0|s_0)P_S(s_1|s_0,a_0)\times\cdots\times P_S(s_t|s_{t-1},a_{t-1})\pi(a_t|s_t).
	\end{align*}
	Hence, $P^{\pi_b}(h_t^{s,a})\rho_t^{\pi_e}(h_t^{s,a})=P^{\pi_e}(h_t^{s,a})$ for any $t=0,...,T$.
	We then have that
	\begin{align*}
    \mathbb{E}_{H\sim \pi_b}[\rho_t^{\pi_e}(H_t^{s,a})R_t]
	&=\mathbb{E}_{H\sim \pi_b}[\rho_t^{\pi_e}(H_t^{s,a})\mathbb{E}_{H\sim \pi_b}[R_t|H_t^{s,a}]]\\
	&=\mathbb{E}_{H\sim \pi_b}[\rho_t^{\pi_e}(H_t^{s,a})\mu(S_t,A_t)]\\
	&=\sum_{h_{t}^{s,a}}P^{\pi_b}(h_t^{s,a})\rho_t^{\pi_e}(h_t^{s,a})\mu(s_t,a_t)\\
	&=\sum_{h_{t}^{s,a}}P^{\pi_e}(h_t^{s,a})\mu(s_t,a_t)\\
	&=\mathbb{E}_{H\sim \pi_e}[\mu(S_t,A_t)]\\
	&=\mathbb{E}_{H\sim \pi_e}[\mathbb{E}_{H\sim \pi_e}[R_t|H_t^{s,a}]]\\
	&=\mathbb{E}_{H\sim \pi_e}[R_t],
	\end{align*}
	where we use $\mathbb{E}_{H\sim \pi_b}[R_t|H_t^{s,a}]=\mathbb{E}_{H\sim \pi_e}[R_t|H_t^{s,a}]=\mu(S_t,A_t)$, and we use $P^{\pi_b}(h_t^{s,a})\rho_t^{\pi_e}(h_t^{s,a})=P^{\pi_e}(h_t^{s,a})$ for the fourth equality.
\end{proof}

\begin{lemma}\label{lemma:step-IS}
	$\mathbb{E}_{H\sim \pi_b}[\sum_{t=0}^T\gamma^t\rho_t^{\pi_e}(H_t^{s,a})R_t]=V^{\pi_e}$.
\end{lemma}
\begin{proof}
	This immediately follows from Lemma \ref{lemma:step-IS:R} and the definition of $V^{\pi_e}$.
\end{proof}

\begin{lemma}\label{lemma:step-IS:S-A}
	For $t=0,...,T$ and for any measurable function $g_t:(\mathcal{S}\times\mathcal{A})^t\rightarrow \mathbb{R}$, $\mathbb{E}_{H\sim \pi_b}[\rho_t^{\pi_e}(H_t^{s,a})g_t(H_t^{s,a})]=\mathbb{E}_{H\sim \pi_e}[g_t(H_t^{s,a})]$.
\end{lemma}
\begin{proof}
	We have that
	\begin{align*}
	\mathbb{E}_{H\sim \pi_b}[\rho_t^{\pi_e}(H_t^{s,a})g_t(H_t^{s,a})]
	&=\sum_{h_{t}^{s,a}}P^{\pi_b}(h_t^{s,a})\rho_t^{\pi_e}(h_t^{s,a})g(h_t^{s,a})\\
	&=\sum_{h_{t}^{s,a}}P^{\pi_e}(h_t^{s,a})g(h_t^{s,a})\\
	&=\mathbb{E}_{H\sim \pi_e}[g_t(H_t^{s,a})],
	\end{align*}
	where we use $P^{\pi_b}(h_t^{s,a})\rho_t^{\pi_e}(h_t^{s,a})=P^{\pi_e}(h_t^{s,a})$ for the second equality.
\end{proof}

\begin{lemma}\label{lemma:step-IS:S}
	For $t=1,...,T$ and for any measurable function $g_t:(\mathcal{S}\times\mathcal{A})^{t-1}\times \mathcal{S}\rightarrow \mathbb{R}$, $\mathbb{E}_{H\sim \pi_b}[\rho_{t-1}^{\pi_e}(H_{t-1}^{s,a})g_t(H_{t-1}^{s,a},S_t)]=\mathbb{E}_{H\sim \pi_e}[g_t(H_{t-1}^{s,a},S_t)]$.
\end{lemma}
\begin{proof}
	We have that
	\begin{align*}
	\mathbb{E}_{H\sim \pi_b}[\rho_{t-1}^{\pi_e}(H_{t-1}^{s,a})g_t(H_{t-1}^{s,a},S_t)]
	&=\sum_{(h_{t-1}^{s,a},s_t)}P^{\pi_b}(h_{t-1}^{s,a})P_S(s_t|s_{t-1},a_{t-1})\times\rho_{t-1}^{\pi_e}(h_{t-1}^{s,a})g(h_{t-1}^{s,a},s_t)\\
	&=\sum_{(h_{t-1}^{s,a},s_t)}P^{\pi_e}(h_{t-1}^{s,a})P_S(s_t|s_{t-1},a_{t-1})g(h_{t-1}^{s,a},s_t)\\
	&=\mathbb{E}_{H\sim \pi_e}[g_t(H_{t-1}^{s,a},S_t)],
	\end{align*}
	where we use $P^{\pi_b}(h_{t-1}^{s,a})\rho_{t-1}^{\pi_e}(h_{t-1}^{s,a})=P^{\pi_e}(h_{t-1}^{s,a})$ for the second equality.
\end{proof}

\section{Proof of Proposition 1}\label{proof}
We use Theorems 3.1 and 3.2 of \citet{Chernozhukov2018} for the proof.
We verify that $\mathbb{E}_{H\sim\pi_b}[\psi(H;\eta)]=V^{\pi_e}$ and that $\psi$ satisfies the Neyman orthogonality condition.
For the first part,
\begin{align*}
\mathbb{E}_{H\sim\pi_b}[\psi(H;\eta)]
&=\sum_{t=0}^T \gamma^t \mathbb{E}_{H\sim\pi_b}[\rho_t^{\pi_e}(H_t^{s,a})(R_t-q_t^{\pi_e}(S_t,A_t))+\rho_{t-1}^{\pi_e}(H_{t-1}^{s,a})\sum_{a\in\mathcal{A}}\pi_e(a|S_t)q_t^{\pi_e}(S_t,a)]\\
&=V^{\pi_e} -\sum_{t=0}^T \gamma^t \{\mathbb{E}_{H\sim\pi_e}[q_t^{\pi_e}(S_t,A_t)]-\mathbb{E}_{H\sim\pi_e}[\sum_{a\in\mathcal{A}}\pi_e(a|S_t)q_t^{\pi_e}(S_t,a)]\}\\
&=V^{\pi_e}-\sum_{t=0}^T \gamma^t \{\mathbb{E}_{H\sim\pi_e}[q_t^{\pi_e}(S_t,A_t)]\mathbb{E}_{H\sim\pi_e}[q_t^{\pi_e}(S_t,A_t)]\}\\
&=V^{\pi_e},
\end{align*}
where we use Lemmas \ref{lemma:step-IS}, \ref{lemma:step-IS:S-A} and \ref{lemma:step-IS:S} for the second equality.

We now show that $\psi$ satisfies the Neyman orthogonality condition.
Let
$$D\rho_t^{\pi_e}(H_t^{s,a})[\tilde \pi_b]\coloneqq \left. \frac{\partial \rho_t^{\pi_e}(H_t^{s,a};\pi_b+r(\tilde \pi_b-\pi_b))}{\partial r}\right|_{r=0}
$$
for any candidate $\tilde \pi_b$.
Note that by Lemmas \ref{lemma:step-IS:S-A} and \ref{lemma:step-IS:S},
\begin{align*}
&\mathbb{E}_{H\sim \pi_b}[-\rho_t^{\pi_e}(H_t^{s,a})\tilde q_t^{\pi_e}(S_t,A_t)+\rho_{t-1}^{\pi_e}(H_{t-1}^{s,a})\sum_{a\in\mathcal{A}}\pi_e(a|S_t)\tilde q_t^{\pi_e}(S_t,a)]\\
=~&\mathbb{E}_{H\sim \pi_e}[-\tilde q_t^{\pi_e}(S_t,A_t)+\sum_{a\in\mathcal{A}}\pi_e(a|S_t)\tilde q_t^{\pi_e}(S_t,a)]\\
=~&\mathbb{E}_{H\sim \pi_e}[-\tilde q_t^{\pi_e}(S_t,A_t)]+\mathbb{E}_{H\sim \pi_e}[\tilde q_t^{\pi_e}(S_t,A_t)]\\
=~&0
\end{align*}
for $t=0,...,T$.
We then have that for any candidate $\tilde \eta=(\tilde \pi_b,\{\tilde q_{t}^{\pi_e}\}_{t=0}^T)$,
\begin{align*}
&\left. \frac{\partial \mathbb{E}_{H\sim\pi_b}[\psi(H;\eta+r(\tilde \eta-\eta))]}{\partial r}\right|_{r=0}\\
=~&\sum_{t=0}^T \gamma^t\mathbb{E}_{H\sim \pi_b}[D\rho_t^{\pi_e}(H_t^{s,a})[\tilde \pi_b](R_t-q_t^{\pi_e}(S_t,A_t))-\rho_t^{\pi_e}(H_t^{s,a})\tilde q_t^{\pi_e}(S_t,A_t)\\
&+D\rho_{t-1}^{\pi_e}(H_{t-1}^{s,a})[\tilde \pi_b]\sum_{a\in\mathcal{A}}\pi_e(a|S_t)q_t^{\pi_e}(S_t,a)+\rho_{t-1}^{\pi_e}(H_{t-1}^{s,a})\sum_{a\in\mathcal{A}}\pi_e(a|S_t)\tilde q_t^{\pi_e}(S_t,a)]\\
=~&\sum_{t=0}^T \gamma^t\mathbb{E}_{H\sim \pi_b}[D\rho_t^{\pi_e}(H_t^{s,a})[\tilde \pi_b](R_t-q_t^{\pi_e}(S_t,A_t))+D\rho_{t-1}^{\pi_e}(H_{t-1}^{s,a})[\tilde \pi_b]\sum_{a\in\mathcal{A}}\pi_e(a|S_t)q_t^{\pi_e}(S_t,a)]\\
=~&\gamma^T\mathbb{E}_{H\sim \pi_b}[D\rho_T^{\pi_e}(H_T^{s,a})[\tilde \pi_b](R_T-q_T^{\pi_e}(S_T,A_T))]\\
&+\sum_{t=0}^{T-1} \gamma^t\mathbb{E}_{H\sim \pi_b}[D\rho_t^{\pi_e}(H_t^{s,a})[\tilde \pi_b](R_t-q_t^{\pi_e}(S_t,A_t)+\gamma\sum_{a\in\mathcal{A}}\pi_e(a|S_{t+1})q_{t+1}^{\pi_e}(S_{t+1},a))].
\end{align*}
Since $\mathbb{E}_{H\sim \pi_b}[R_T|H_T^{s,a}]=\mu(S_T,A_T)$ and $q_{T_e}^\pi(S_T,A_T)=\mu(S_T,A_T)$, the first term is zero by the law of iterated expectations.
The second term is also zero, since for $t=0,...,T-1$,
\begin{align*}
&\mathbb{E}_{H\sim \pi_b}[D\rho_t^{\pi_e}(H_t^{s,a})[\tilde \pi_b](R_t+\gamma\sum_{a\in\mathcal{A}}\pi_e(a|S_{t+1})q_{t+1}^{\pi_e}(S_{t+1},a))]\\
=~&\mathbb{E}_{H\sim \pi_b}[D\rho_t^{\pi_e}(H_t^{s,a})[\tilde \pi_b]\mathbb{E}_{H\sim \pi_b}[R_t+\gamma\sum_{a\in\mathcal{A}}\pi_e(a|S_{t+1})q_{t+1}^{\pi_e}(S_{t+1},a)|H_t^{s,a}]]\\
=~&\mathbb{E}_{H\sim \pi_b}[D\rho_t^{\pi_e}(H_t^{s,a})[\tilde \pi_b](\mu(S_t,A_t)+\gamma\sum_{s\in \mathcal{S}}P_S(s|S_t,A_t)\sum_{a\in\mathcal{A}}\pi_e(a|s)q_{t+1}^{\pi_e}(s,a))]\\
=~&\mathbb{E}_{H\sim \pi_b}[D\rho_t^{\pi_e}(H_t^{s,a})[\tilde\pi_b]q_t^{\pi_e}(S_t,A_t)],
\end{align*}
where we use the recursive formulation of $q_t^\pi$ for the last equality.

The convergence results then follow from Theorems 3.1 and 3.2 of \citet{Chernozhukov2018}.
\qed


\end{document}